\documentclass[11pt,twoside]{article}
\usepackage[margin=1in]{geometry}

\usepackage{amsmath,amssymb,amsthm,bbm,natbib}
\usepackage{verbatim,makeidx}
\usepackage{algorithm}
\usepackage{algorithmic}
\usepackage{libertine,multirow}
\usepackage[libertine]{newtxmath}
\usepackage{color}
\usepackage{colortbl}
\usepackage{dsfont}
\usepackage{caption}
\usepackage{subcaption}
\usepackage{graphicx}
\usepackage{verbatim}
\usepackage{url}
\usepackage{hyperref}
\usepackage{pbox}
\usepackage{authblk}
\usepackage{bm}

\allowdisplaybreaks

\newcounter{ALC@tempcntr}

\theoremstyle{plain}
\newtheorem{theorem}{Theorem}
\newtheorem{proposition}{Proposition}
\newtheorem{lemma}{Lemma}

\newtheorem{remark}{Remark}

\theoremstyle{definition}

\theoremstyle{remark}

\newcommand{\beq}{\begin{eqnarray}}
\newcommand{\eeq}{\end{eqnarray}}

\newcommand{\field}[1]{\mathbb{#1}}

\newcommand{\R}{\field{R}}

\newfont{\bbb}{msbm10 scaled 500}

\newfont{\bb}{msbm10 scaled 1100}

\newcommand{\av}{{\bf a}}
\newcommand{\bv}{{\bf b}}
\newcommand{\cv}{{\bf c}}

\newcommand{\ev}{{\bf e}}
\newcommand{\fv}{{\bf f}}

\newcommand{\hv}{{\bf h}}

\newcommand{\mv}{{\bf m}}

\newcommand{\uv}{{\bf u}}
\newcommand{\wv}{{\bf w}}
\newcommand{\vv}{{\bf v}}
\newcommand{\xv}{{\bf x}}
\newcommand{\yv}{{\bf y}}
\newcommand{\zv}{{\bf z}}

\newcommand{\Fc}{{\cal F}}

\newcommand{\Lc}{{\cal L}}

\newcommand{\Nc}{{\cal N}}

\newcommand{\Rc}{{\cal R}}
\newcommand{\Sc}{{\cal S}}
\newcommand{\Tc}{{\cal T}}
\newcommand{\Uc}{{\cal U}}

\newcommand{\Vc}{{\cal V}}
\newcommand{\Xc}{{\cal X}}
\newcommand{\Yc}{{\cal Y}}

\DeclareMathOperator{\rank}{rank}

\DeclareMathOperator{\relu}{ReLU}

\newcommand{\remove}[1]{}

\newcommand{\avg}{{\mathbb E}}
\newcommand{\pr}{{\mathbb P}}

\newcommand\reals{{\mathbb R}}

\theoremstyle{definition}

\theoremstyle{remark}

\newcommand{\latexe}{{\LaTeX\kern.125em2%
                      \lower.5ex\hbox{$\varepsilon$}}}

\chardef\bslash=`\\	

\makeatletter		
\def\square{\RIfM@\bgroup\else$\bgroup\aftergroup$\fi
\vcenter{\hrule\hbox{\vrule\@height.6em\kern.6em\vrule}
\hrule}\egroup}\makeatother\makeindex

\definecolor{OXO-emph}{RGB}{153,0,0}

\DeclareMathAlphabet{\mathpzc}{OT1}{pzc}{m}{it}

\title{Representation Learning and Recovery in the ReLU Model}

\author[1]{Arya Mazumdar}
\author[1,2]{Ankit Singh Rawat}
\affil[1]{\normalsize College of Information and Computer Sciences, University of Massachusetts Amherst, Amherst, MA 01003, USA.}
\affil[2]{\normalsize Research Laboratory of Electronics, MIT, Cambridge, MA 02139, USA.\newline E-mail: arya@cs.umass.edu, asrawat@mit.edu.}

\begin{document}

\maketitle


\begin{abstract}
Rectified linear units, or ReLUs, have become the preferred activation function for artificial neural networks. In this paper we consider two basic learning problems assuming that the underlying data follow a generative model based on a ReLU-network -- a neural network with ReLU activations. As a primarily theoretical study, we limit ourselves to a single-layer network. The first problem we study corresponds to dictionary-learning in the presence of nonlinearity (modeled by the ReLU functions). Given a set of observation vectors $\mathbf{y}^i \in \mathbb{R}^d, i =1, 2, \dots , n$, we  aim to recover $d\times k$ matrix $A$ and the latent vectors $\{\mathbf{c}^i\} \subset \mathbb{R}^k$ under the model $\mathbf{y}^i = \mathrm{ReLU}(A\mathbf{c}^i +\mathbf{b})$, where $\mathbf{b}\in \mathbb{R}^d$ is a random bias. We show that it is possible to recover the column space of $A$ within an error of $O(d)$ (in Frobenius norm) under certain conditions on the probability distribution of $\mathbf{b}$.

The second problem we consider is that of robust recovery of the signal in the presence of outliers, i.e., large but sparse noise. In this setting we are interested in recovering the latent vector $\mathbf{c}$ from its noisy nonlinear sketches of the form $\mathbf{v} = \mathrm{ReLU}(A\mathbf{c}) + \mathbf{e}+\mathbf{w}$, where $\mathbf{e} \in \mathbb{R}^d$ denotes the outliers with sparsity $s$ and $\mathbf{w} \in \mathbb{R}^d$ denote the dense but small noise. This line of work has recently been studied (Soltanolkotabi, 2017) without the presence of outliers. For this problem, we show that a generalized LASSO algorithm is able to recover the signal $\mathbf{c} \in \mathbb{R}^k$ within an $\ell_2$ error of $O(\sqrt{\frac{(k+s)\log d}{d}})$ when $A$ is a random Gaussian matrix. 
\end{abstract}


\section{Introduction}
\label{sec:sys}
Rectified Linear Unit (ReLU) is a basic nonlinear function defined to be $\relu:\reals \to \reals_{+}\cup \{0\}$ as $\relu(x) \equiv \max(0,x)$. For any matrix $X$, $\relu(X)$ denotes the  matrix obtained by applying the ReLU function on each of the coordinates of the matrix $X$. ReLUs are  building blocks of many nonlinear data-fitting problems based on deep neural networks (see, e.g., \cite{Soltan17} for a good exposition).

Let $\Yc \subset \R^{d}$ be a collection of message vectors that are of interest to us. Depending on the application at hand, the message vectors, i.e., the constituents of $\Yc$, may range from images, speech signals, network access patterns to user-item rating vectors and so on. We assume that the message vectors satisfy a generative model, where each message vector can be approximated by a map $g :\R^{k} \to \R^{d}$ from the latent space to the ambient space, i.e., for each $\yv \in \Yc$,
\begin{align}
\label{eq:nl_model}
\yv \approx g(\cv)~~\text{for some}~\cv \in \R^{k}.
\end{align}
Motivated by the recent results on developing the generative models for various real-life signals (see e.g., \cite{GAN14, KW14, BJPD17}), the non-linear maps $g$ that take the following form warrant special attention.
\begin{align}
\label{eq:nn_model}
g(\cv) = h\left(A_{l}h(A_{l-1}\cdots h(A_1\cv)\cdots ))\right),
\end{align} 
i.e., $g$ is the function corresponding to an $l$-layer neural network with the activation function $h$. Here, for $i \in [l]$, $A_i \in R^{d_i \times d_{i-1}}$ with $d_0 = k$ and $d_l = d$, denotes the weight matrix for the $i$-th layer of the network. In the special case, where the activation function $h$ is the $\relu$ function, the message vectors of the interest satisfy the following.
\begin{align}
\label{eq:relu_model}
\yv \approx \relu\left(A_{l}\relu(A_{l-1}\cdots \relu(A_1\cv + \bv_1)\cdots + \bv_{l-1} ) + \bv_{l})\right),
\end{align} 
where, for $i \in [l]$, $\bv_{i} \in \R^{d_i}$ denotes the biases of the neurons (or output units) at the $i$-th layer of the network. 

The specific generative model in \eqref{eq:relu_model} raises multiple interesting questions that play fundamental role in understanding the underlying data and designing systems and algorithms for processing the data. Two such most basic questions are as follows:
\begin{enumerate}
\item\textbf{Learning the representation:}~Given the observations $\{\yv^{1},\yv^{1},\ldots, \yv^{n}\} \subset \R^{d}$ from the model (cf.~\eqref{eq:relu_model}), recover the parameters of the model, i.e., $\{\hat{A}_{t}\}_{t \in [l]}$, and $\{\cv^{1},\cv^{2},\ldots, \cv^{n}\} \subset \R^{k}$ such that 
\begin{align}\label{eq:dict}
\yv^{i} \approx \relu\left(\hat{A}_{l}\relu(\hat{A}_{l-1}\cdots \relu(\hat{A}_1\cv^{i} + \bv_1)\cdots + \bv_{l-1} ) + \bv_{l})\right)~\forall~i \in [n].
\end{align}
Note that this question is different from training the model, in which case the set $\{\cv^{1},\cv^{2},\ldots, \cv^{n}\}$ is known (and possibly chosen accordingly).
\item\textbf{Recovery of the signal in the presence of errors:}~Given the erroneous (noisy) version of a vector generated by the model (cf.~\eqref{eq:relu_model}), denoise the observation or recover the latent vector. Formally, given
\begin{align}
\label{eq:noisy_obs}
\vv = \yv + \ev + \wv = \relu\left(A_{l}\relu(A_{l-1}\cdots \relu(A_1\cv + \bv_1)\cdots + \bv_{l-1} ) + \bv_{l})\right) + \ev + \wv
\end{align}
and the knowledge of model parameters, obtain $\hat{\yv} \in \R^{d}$ or $\hat{\cv} \in \R^{k}$ such that $\|\yv - \hat{\yv}\|$ or $\|\cv - \hat{\cv}\|$ is small, respectively. In \eqref{eq:noisy_obs}, $\ev$ and $\wv$ correspond to outliers, i.e., large but sparse errors, and (dense but small) noise, respectively. 
\end{enumerate}
Apart from being closely related, one of our main motivations behind studying these two problems together comes from the recent work on associative memory \cite{KSSV14, mazumdar2015associative,mazumdar2017assciative}. An associative memory consists of a learning phase, where a generative model is learned from a given dataset; and a recovery phase, where given a noisy version of a data point generated by the generative model, the noise-free version is recovered with the help of the knowledge of the generative model.

There have been a recent surge of interest in learning ReLUs, and the above two questions are of basic interest even for a single-layer network (i.e., nonlinearity comprising of a single ReLU function). It is conceivable that understanding the behavior of a single-layer network would allow one to use some `iterative peeling off' technique to develop a theory for multiple layers. In \cite{GoelKKT17}, the problem of recovering $\relu$-model under Reliable Agnostic learning model of \cite{KKM12} is considered. Informally speaking, under very general distributional assumptions (the rows of $A$ are sampled from some distribution), given $A$ and $\yv = \relu(A\cv)$, \cite{GoelKKT17} propose an algorithm that recovers a hypothesis which has an error-rate (under some natural loss function defined therein) of $\epsilon$ with respect to the true underlying $\relu$-model. Moreover, the algorithm runs in time polynomial in $d$ and exponential in ${1}/{\epsilon}$. As opposed to this, given $A$ and the corresponding output of the $\relu$-network $\yv = \relu(A\cv + \bv)$, we focus on the problem of recovering $\cv$ itself. Here, we note that the the model considered in \cite{GoelKKT17} does not consider the presence of outliers.

 \cite{Soltan17} obtained further results on this model under somewhat different learning guarantees. Assuming that the entries of the matrix $A$ to be i.i.d. Gaussian, \cite{Soltan17} show that with high probability a gradient descent algorithm recovers $\cv$ within some precision in terms of $\ell_2$-loss: the relative error decays exponentially with the number of steps in the gradient descent algorithm. The obtained result is more general as it extends to constrained optimizations in the presence of some regularizers (for example, $\cv$ can be restricted to be a sparse vector, etc.).  

However both of these works do not consider the presence of outliers (sparse but large noise) in the observation. The sparse noise is quite natural to assume, as many times only partial observations of a signal vector are obtained. {The ReLU model with outliers as considered in this paper can be thought of as a nonlinear version of the problem of recovering $\cv$ from linear observations of the form $\vv = A\cv+ \ev$, with $\ev$ denoting the outliers. This problem with linear observations was studied in the celebrated work of \cite{can05}. We note that the technique of \cite{can05} does not extend to the case when there is a dense (but bounded) noise component present.} Our result in this case  is a natural generalization and complementary to the one in \cite{Soltan17} in that 1) we present a recovery method which is robust to outliers and 2) instead of analyzing gradient descent we directly analyze the performance of the minimizer of our optimization program (a generalized LASSO) using the ideas from \cite{PlanV16, NgTran13}.
 
On the other hand, to the best of our knowledge, the representation learning problem for single-layer networks has not been studied as such. The representation learning problem for single-layer ReLUs bears some similarity with matrix completion problems, a fact we greatly exploit later. In low rank matrix completion,  a matrix $M= AX$ is visible only partially, and the task is to recover the unknown entries by exploiting the fact that it is low rank. In the case of \eqref{eq:dict}, we are more likely to observe the positive entries of the matrix $M$, which, unlike a majority of matrix completion literature, creates the dependence between the matrix $M$ and the sampling procedure.\\
 
\noindent\textbf{Main result for representation learning.}~We assume to have observed $d\times n$ matrix $Y= \relu(AC+\bv\otimes \mathbf{1}^T)$ where $A$ is a $d\times k$ matrix, $C$ is a $k \times n$ matrix, both unknown,  $\bv\in \reals^d$ is a random i.i.d. bias, and $\otimes$ denote the Kronecker product\footnote{This is to ensure that the bias is random, but does not change over different observation of the data samples.}. We show that a relaxed maximum-likelihood method guarantees the recovery of the matrix $AC$ with an error in Frobenius norm at most $O(\sqrt{d})$ with high probability (see Theorem~\ref{thm:MLperformance_real} for the formal statement). Then leveraging a known result for recovering column space of a perturbed matrix (see Theorem.~\ref{thm:yu} in the appendix), we show that it is possible to also recover the column space of $A$ with similar guarantee.  

The main technique that we use to obtain this result is inspired by the recent work on matrix completion by \cite{DPBW14}. One of the main challenges that we face in recovery here is that while an entry of the matrix $Y$ is a random variable (since $\bv$ is a random bias), whether that is being observed or being cut-off by the ReLU function (for being negative) depends on the value of the entry itself. In general matrix completion literature, the entries of the matrix being observed are sampled i.i.d. (see, for example, \cite{CR09, KMS10,C15} and references therein). For the aforementioned reason we cannot use most of these results off-the-shelf. However, similar predicament is (partially) present in \cite{DPBW14}, where entries are quantized while being observed.

Similar to \cite{DPBW14}, the tools that prove helpful in this situation are the symmetrization trick and the contraction inequality \cite{LedT13}. However, there are crucial difference of our model from \cite{DPBW14}: in our case the bias vector, while random, do not change over observations. This translates to less freedom during the transformation of the original matrix to the observed matrix, leading to dependence among the elements in a row. Furthermore, the analysis becomes notably different since the positive observations are not quantized. \\

\noindent\textbf{Main result for noisy recovery.}~We plan to recover $\cv\in \reals^k$ from observations $\vv = \relu(A\cv +\bv) +  \ev +\wv$, where $A$ is a $d \times k$ standard i.i.d. Gaussian matrix, $\ev \in \reals^d$ is the vector containing outliers (sparse noise) with $\|\ev\|_0 \le s$, and $\wv\in \reals^d $ is bounded dense noise such that $\|\wv\|_\infty \le \delta$. To recover $\cv$ (and $\ev$) we employ the LASSO algorithm, which is inspired by the work of \cite{PlanV16} and \cite{NgTran13}. In particular, \cite{PlanV16} recently showed that a signal can be provably recovered (up to a constant multiple) from its nonlinear Gaussian measurements via the LASSO algorithm by treating the measurements as linear observations. In the context of $\relu$ model, for outlier-free measurements $\vv = \relu(A\cv +\bv) + \wv$, it follows from \cite{PlanV16} that LASSO algorithm outputs $\mu\cv$ as the solution with $\mu = \avg(g\cdot \relu(g + b))$, where $g$ is a Gaussian random variable and $b$ is a random variable denoting bias associated with the $\relu$ function. We show that this approach guarantees with high probability recovery of $\cv$ within an $\ell_2$ error of $O\big(\sqrt{\frac{(k+s)\log d}{d}}\big)$ even when the measurements are corrupted by outliers $\ev$. This is achieved by jointly minimizing the square loss over $(\cv, \ev)$ after treating our measurements as linear measurements $\vv' = A\cv + \ev +  \wv$ and adding an $\ell_1$ regularizer to the loss function to promote the sparsity of the solution for $\ev$ (we also recover $\ev$, see Theorem~\ref{thm:rlasso} for a formal description).\\

\noindent\textbf{Organization.}~The paper is organized as follows. In section~\ref{sec:bg}, we describe some of the  notations used throughout the paper and  introduce the some technical tools that would be useful to prove our main results. In the same section (subsection~\ref{sec:system}), we provide the formal models of the problem we are studying. In section~\ref{sec:learning}, we provide detailed proofs for our main results on the representation learning problem (see, Theorem~\ref{thm:MLperformance_real}). Section~\ref{sec:robust} contains the proofs and the techniques used  for the recovery problem in the presence of outliers (see, Theorem~\ref{thm:rlasso}).


\section{Notations and Technical Tools}
\label{sec:bg}

\subsection{Notation}
\label{sec:notation}
For any positive integer $n$, define $[n]\equiv \{1,2,\ldots,n\}$. Given a matrix $M \in \R^{d \times n}$, for $(i, j) \in [d] \times [n]$, $M_{i,j}$ denotes the $(i, j)$-th entry of $M$. For $i \in [d]$, $\mv_{i} = (M_{i,1}, M_{i,2},\ldots, M_{i,n})^T$ denotes the vector containing the elements of the $i$-th row of the matrix $M$. Similarly, for $j \in [n]$, $\mv^{j} = (M_{1,j}, M_{2,j},\ldots, M_{d,j})^T$ denotes the $j$-th column of the matrix $M$. Recall that the function ${\rm ReLU}:\R \rightarrow \R_{+}\cup\{0\}$ takes the following form.
\begin{align}
{\rm ReLU}(x) = \max(x, 0).
\end{align}
For a matrix $X \in \R^{d \times n}$, we use ${\rm ReLU}(X)$ to denote the $d\times n$ matrix obtained by applying the {\rm ReLU} function on each of the entries of the matrix $X$. For two matrix $A$ and $B$, we use $A\otimes B$ to represent the Kronecker product of $A$ and $B$.

Given a matrix $P \in \R^{d \times n}$, $$\|P\|_F = \sqrt{\sum_{(i,j) \in [d]\times[n]}P_{i,j}^2}$$ denotes its Frobenius norm. Also, let $\|P\|$ denote the $\ell_2$ operator norm of $P$, i.e. the maximum singular value of $P$. We let $\|P\|_\ast$ denote the nuclear norm of $P$. Similar to \cite{DPBW14}, we define a flatness parameter associated with a function $f:\R \to [0,1]$:
\begin{align}
\label{eq:flat}
\beta_{\alpha}(f) := \inf_{|x| \le \alpha} \frac{|f'(x)|^2}{4|f(x)|}.
\end{align}
$\beta_{\alpha}$ quantifies how flat $f$ can be in the interval $[-\alpha, \alpha]$. We also define a Lipschitz parameter $L_\alpha(f)$ for a function $f$ as
follows:
\begin{align}
\label{eq:steep}
L_\alpha(f) :=  \max\Big\{ \sup_{|x| \le \alpha}\frac{f(x)}{\int_{-\infty}^xf(y) dy}, \sup_{|x| \le \alpha} \frac{f'(x)}{f(x)}\Big\}.
\end{align}

\subsection{Techniques to bound the supremum of an empirical process}
In the course of this paper, namely in the representation learning part, we use the key tools of symmetrization and contraction to bound the supremum of an empirical process following the lead of \cite{DPBW14} and the analysis of generalization bounds in the statistical learning literature. In particular, we need the following two statements.
\label{sec:contraction}
\begin{theorem}[\bf Symmetrization of expectation]
\label{thm:symm}
Let $X_1, X_2,\ldots, X_n$ be $n$ independent RVs taking values in $\Xc$ and $\Fc$ be a class of~$\R$-valued functions on $\Xc$. Furthermore, let $\varepsilon_1, \varepsilon_2,\ldots, \varepsilon_n$ be $n$ independent Rademacher RVs. Then, for any $t \geq 1$,
\begin{align}
\avg\left(\sup_{f \in \Fc}\Big\vert\sum_{i = 1}^{n}\big(f(X_i) - \avg f(X_i)\big)\Big\vert^{t}\right) \leq 2^{t}\avg\left(\sup_{f \in \Fc}\Big\vert\sum_{i = 1}^{n}\varepsilon_if(X_i)\Big\vert^{t}\right).
\end{align}
\end{theorem}

\begin{theorem}[\bf Contraction inequality \cite{LedT13}]
\label{thm:contraction}
Let $\varepsilon_1, \varepsilon_2,\ldots, \varepsilon_d$ be $d$ independent Rademacher RVs and  $f: \R_{+} \rightarrow \R_{+}$ be a convex and increasing function. Let $\zeta_{i}:\R \rightarrow \R$ be an $L$-Lipschitz functions, i.e.,
$$
|\zeta_i(a) - \zeta_i(b)| \leq L |a - b|,
$$ 
which satisfy $\zeta_i(0) = 0$. Then, for any $T \subseteq \R^n$, 
\begin{align}
\label{eq:contraction}
\avg f\left(\frac{1}{2} \sup_{t \in T} \Big\vert \sum_{i = 1}^{d} \varepsilon_i \zeta_i(t_i)\Big\vert \right) \leq \avg f \left(L \cdot \sup_{t \in T}\Big\vert \sum_{i = 1}^{d}\varepsilon_it_i\Big\vert\right).
\end{align}
\end{theorem}


\subsection{System Model}
\label{sec:system}

We focus on the problems of learning the representation and recovery of the signal in the presence of errors when the signal is assumed to be generated using a {\em single layer} $\relu$-network. The models of learning representations and recovery is described below. \\

\noindent\textbf{Model for learning representations.}~We assume that a signal vector of interest $\yv$ satisfies 
\begin{align}
\yv = \relu(A\cv + \bv),
\end{align}
where $A \in \R^{d \times k}$ and $\bv \in \R^{d}$ correspond to the weight ({\em generator}) matrix and the bias vector, respectively. 

As for the problem of representation learning, we are given $n$ message vectors that are generated from the underlying single-layer model. For $j \in [n]$, the $j$-th signal vector is defined as follows.
\begin{align}
\yv^{j} = {\rm ReLU}\left(A\cv^{j} + \bv\right) \in \R^{d}.
\end{align}
We define the $d \times n$ observation matrix
\begin{align}
Y= \big[\begin{array}{cccc}
\yv^{1} & \yv^{2} & \cdots & \yv^{n}
\end{array}\big].
\end{align}
Similarly, we define the $k \times  n$ coefficient matrix
\begin{align}
C = \big[\begin{array}{cccc}
\cv^{1} & \cv^{2} & \cdots & \cv^{n}
\end{array}\big].
\end{align}
With this notion, we can concisely represent the $n$ observation vectors as
\begin{align}
\label{eq:sigmodel}
Y = {\rm ReLU}\left(AC + \bv\otimes \mathbf{1}^T\right) = {\rm ReLU}\left(M + \bv\otimes \mathbf{1}^T\right) ,
\end{align}
where $\mathbf{1} \in \R^n$ denotes the all-ones vector. 

We assume that the bias vector $\bv$ is a random vector comprising of i.i.d. coordinates with each coordinate being copies of a random variable $B$ distributed according to probability 
density function $p(\cdot)$. \\

\noindent\textbf{Model for recovery.} For the recovery problem, we are given a vector $\vv \in \reals^d$, which is obtained by adding noise to a valid signal vector $\yv\in\reals^d$ that is well modeled by a single-layer $\relu$-network, with the matrix $A\in \reals^{d \times k}$ and bias $\bv \in \reals^d$. In particular, for some $\cv \in \reals^k$ we have,
\begin{align}
\label{eq:obsmodel}
\vv = \yv + \ev + \wv = \relu(A\cv + \bv) + \ev + \wv,
\end{align}
where $\wv \in \reals^d$ denotes the (dense) noise vector with bounded norm. On the other hands, the vector $\ev\in \reals^d$ contains (potentially) large corruptions, also referred to as sparse errors or outliers (we assume, $\|\ev\|_0 \le s$). The robust recovery problem in $\relu$-networks corresponds to obtaining an estimate $\hat{\cv}$ of the true latent vector $\cv$ from the {\em corrupt} observation vector $\yv$ such that the distance between $\hat{\cv}$ and $\cv$ is small. A related problem of denoising in the presence of outliers only focuses on obtaining an estimate $\yv$ which is close to the true signal vector $\yv$. For this part, we focus on the setting where the weight matrix $A$ is a random matrix with i.i.d. entries, where each entry of the matrix is distributed according to standard Gaussian distribution. Furthermore, another crucial assumption is that the Hamming error is {\em oblivious} in its nature, i.e., the error vector is not picked in an adversarial manner given the knowledge of $A$ and $\cv$\footnote{It is an interesting problem to extend our results to a setting with adversarial errors. However, we note that this problem is an active area of research even in the case of linear measurement, i.e, $\yv = A\cv + \ev + \wv$~\cite{BJK15, BJKK17}. We plan to explore this problem in future work.}.

\section{Representation learning in a single-layer $\relu$-network}
\label{sec:learning}

In the paper, we employ the natural approach to learn the underlying weight matrix $A$ from the observation matrix $Y$. As the network maps a lower dimensional coefficient vector $\cv \in \R^{k}$   to obtain a signal vector 
\begin{align}
\yv = \relu(A\cv + \bv)
\end{align}
in dimension $d > k$, the matrix $M = AC$ (cf.~\eqref{eq:sigmodel}) is a low-rank matrix as long as $k < \min\{d, n\}$. In our quest of recovering the weight matrix $A$, we first focus on estimating the matrix $M$, when given access to $Y$. This task can be viewed as estimating a {\em low-rank} matrix from its partial (randomized) observations. Our work is inspired by the recent 
 work of~\cite{DPBW14} on $1$-bit matrix completion. However, as we describe later, the crucial difference of our model from the model of \cite{DPBW14} is that the bias vector $\bv$ does not change over observations in our case. Nonetheless we describe the model and main results of $1$-bit matrix completion below to underscore the key ideas.\\
 
\noindent\textbf{$1$-bit matrix completion.}~In \cite{DPBW14}, the following observation model is considered. Given a low-rank matrix $M$ and a differentiable function $\psi: \R \rightarrow [0,1]$, the matrix $Z \in \{0, 1\}^{d \times n}$ is assumed to be generated as follows\footnote{The authors assume that the entries of $Z$ take values in the set $\{+1, -1\}$. In this paper we state the equivalent model where the binary alphabet is $\{0, 1\}$.}.
\begin{align}
\label{eq:z_dpbw}
Z_{i, j} = \begin{cases}
1 & \text{with probability}~\psi(M_{i,j}), \\
0 & \text{with probability}~1 - \psi(M_{i,j}).
\end{cases}
\end{align}
Furthermore, one has access to only those entries of $Z$ that are indexed by the set $\Omega \subset [d] \times [m]$, where the set $\Omega$ is generated by including each $(i,j) \in [d] \times [n]$ with certain probability $p$. Given the observations $Z_{\Omega}$, the likelihood function associated with a matrix $X \in \R^{d \times n}$ takes the following form\footnote{Throughout this paper, $\log$ represents the natural logarithm.}. 
\begin{align}
\label{eq:ll_dpbw}
\Lc_{\Omega, Z}(X) = \sum_{(i,j) \in \Omega}\left(\mathbbm{1}_{\{Z_{i,j} =1\}}\log\left(\psi(X_{i,j})\right) + \mathbbm{1}_{\{Z_{i,j} = 0\}}\log\left(1 - \psi(X_{i,j})\right)\right).
\end{align}

Now, in order to estimate the matrix $M$ with bounded entries from $Z_{\Omega}$, it is natural to maximize the log-likelihood function  (cf.~\eqref{eq:ll_dpbw}) under the constraint that the matrix $M$ has rank $k$. 
\begin{equation}
\begin{aligned}
\label{eq:ll_simple}
&\underset{X \in \R^{d \times n}}{\text{maximize}}
& & \Lc_{\Omega, Z}(X) \\
& \text{subject to}
& & \rank(X) = r \\
& & &  \|X\|_{\infty} \leq \gamma,
\end{aligned}
\end{equation}
where the last constraint is introduced to model the setting where (w.h.p.) the observations are assumed to have bounded coordinates. We note that such assumptions indeed hold in many observations of interests, such as images. Note that the formulation in \eqref{eq:ll_simple} is clearly non-convex due to the rank constraint. Thus, \cite{DPBW14} propose the following program.
\begin{equation}
\begin{aligned}
\label{eq:ml_dpbw}
&\underset{X \in \R^{d \times n}}{\text{maximize}}
& & \Lc_{\Omega, Z}(X) \\
& \text{subject to}
& & \|X\|_{\ast} \leq \alpha \sqrt{rmn},\\
& & & \|X\|_{\infty} \leq \gamma.
\end{aligned}
\end{equation}
{Note that the constraint $\|X\|_{\ast} \leq \alpha \sqrt{rmn}$ is a convex-relaxation of the non-convex constraint $\rank(X) \leq r$, which is required to ensure that the program in \eqref{eq:ml_dpbw} outputs a low-rank matrix.} Let $\widehat{M}$ be the output of the program in \eqref{eq:ml_dpbw}. \cite{DPBW14} obtain the following result to characterize the quality of the obtained solution $\widehat{M}$.
\begin{proposition}[{\cite[Theorem~A.1]{DPBW14}}]
\label{prop:DPBW_main}
Assume that $\|M\|_{\ast} \leq \alpha \sqrt{rmn}$ and $\|M\|_{\infty} \leq \gamma$. Let $Z$ be as defined in \eqref{eq:z_dpbw}. Then, for absolute constants $C_1$ and $C_2$, with probability at least $1 - {C_1}/{(d + n)}$, the solution of \eqref{eq:ml_dpbw} $\widehat{M}$ satisfies the following:
\begin{align}
\frac{1}{dn} \|M - \widehat{M}\|_{F} \leq C_2\alpha \sqrt{\frac{r(m+n)}{\avg(|\Omega|)}} \cdot \sqrt{1 + \frac{(m+n)\log(mn)}{\avg(|\Omega|)}},
\end{align}
where the constant $C_2$ depends on the flatness and steepness of the function $\psi.$
\end{proposition}

\vspace{1em}

\noindent\textbf{Learning in a single layer $\relu$ and $1$-bit matrix completion:  Main differences.}~Note that the problem of estimating the matrix $M$ from $Y$ is  related to the problem of $1$-bit matrix completion as defined above. Similar to the $1$-bit matrix completion setup, the observation matrix $Y$ is obtained by transforming the original matrix $M$ in a probabilistic manner, which is dictated by the underlying distribution of the bias vector $\bv$. In particular, we get to observe the entire observation matrix, i.e., $\Omega = [d]\times [n]$. 

However, there is key difference between these two aforementioned setups. The $1$-bit matrix completion setup studied in \cite{DPBW14} (in fact, most of the literature on non-linear matrix completion~\cite{GBW15}) assume that each entry of the original matrix $M$ is independently transformed to obtain the observation matrix $Y$. In contrast to this, such independence in absent in the single-layer $\relu$-network. In particular, for $i \in [d]$, the $i$-th row of the matrix  $Y$ is obtained from the corresponding row of $M$ by utilizing the shared randomness defined by the bias $b_i$. Note that the bias associated with a coordinate of the observed vector in our generative model should not vary across observation vectors. This prevents us from  applying the known results to the problem of estimating $M$ from $Y$. However, as we show in the remainder of this paper that the well-behaved nature of the $\relu$-function allows us to deal with the dependence across the entries of a row in $Z$ and obtain the recovery guarantees that are similar to those described in Proposition~\ref{prop:DPBW_main}.

\subsection{Representation learning from rectified observations}
\label{sec:ML_real}

We now focus on the task of recovering the matrix $M$ from the observation matrix $Y$. Recall that, under the single-layer $\relu$-network, the observation matrix $Y$ depends on the matrix $M$ as follows. 
\begin{align}
Y = \relu(M + \bv \otimes \mathbbm{1}^T).
\end{align}
For $i \in [d]$, we define $\Nc_{Y}(i) \subseteq [n]$ as the set of positive coordinates of the $i$-th row of the matrix $Y$, i.e.,
\begin{align}
\Nc_{Y}(i) = \{j \in [n] : Y_{i, j} > 0 \}~~\text{and}~~~N_{Y,i} =  \vert\Nc_{Y}(i)\vert.
\end{align}
Note that, for $i \in [d]$, the original matrix $M$ needs to satisfy the following requirements.
\begin{align}
\label{eq:real_obs1}
M_{i,j} + b_i = Y_{i,j}~~\text{for}~j \in \Nc_{Y}(i)
\end{align}
and
\begin{align}
\label{eq:real_obs2}
M_{i,j} + b_i < 0~~\text{for}~j \in \overline{\Nc}_{Y}(i):=[n]\backslash\Nc_{Y}(i).
\end{align}
Given the original matrix $M$, for $i \in [d]$ and $j \in [n]$, let $M_{i,(j)}$ denote the $j$-th largest element of the $i$-th row of $M$, i.e., for $i \in [d]$,
\begin{align*}
M_{i, (1)} \geq M_{i, (2)} \geq \cdots \geq M_{i, (n)}.
\end{align*}
It is straightforward to verify from \eqref{eq:real_obs1} that $\Nc_{Y}(i)$ denotes the indices of $N_{Y, i}$ largest entries of $M$. Furthermore, whenever $N_{Y,i} = s_i \in [n]$, we have
\begin{align}
b_i = Y_{i, (1)} - M_{i,(1)} = \cdots = Y_{i, (s_i)} - M_{i,(s_i)}.
\end{align}
Similarly, it follows from \eqref{eq:real_obs2} that whenever we have $N_{Y,i} = 0$, then $b_i$ satisfies the following.
\begin{align}
b_i \in (-\infty, - \max_{j \in [n]}M_{i,j}) = (-\infty, -M_{i,(1)}).
\end{align}
Based on these observation, we define the set of matrices $\Xc_{Y, \nu, \gamma} \subset \R^{d \times n}$ as
\begin{align}
\label{eq:Xc_real}
\Xc_{Y, \nu, \gamma} = \big\{X : \|X\|_{\infty} \leq \gamma;~Y_{i,(1)} - X_{i,(2)} = \cdots = Y_{i,(s_i)} - X_{i,(s_i)};~\text{and}~X_{i,(s_i)} \geq \max_{j \in \overline{\Nc}_{Y}(i)}X_{i,j} + \nu~\forall i \in [d] \big\}.
\end{align}
Recall that,   $p : \R \to \R$ denote the probability density function of each bias RV. We can then write the likelihood that a matrix $X \in \Xc_{Z, \nu, \gamma}$ results into the observation matrix $X$ as follows.
\begin{align}
\pr(Y | X) = \prod_{i \in [d]} \pr(\yv_i | \xv_i),
\end{align}
where, for $i \in [d]$,
\begin{align}
\label{eq:partial_ll}
\pr(\yv_i | \xv_i) &=  \mathbbm{1}_{\{N_{Y,i} = 0\}}\cdot \pr(b_i \leq - \max_{j \in [n]}X_{i,j}) +  \sum_{s = 1}^{n} \mathbbm{1}_{\{N_{Y,i} = s\}}\cdot p(b_i = Y_{i, (s)} - X_{i,(s)}).
\end{align}
By using the notation $F(x_1, x_2) = \pr(-x_1 \leq B \leq -x_2)$ and $X^{\ast}_{i} = \max_{j \in [n]}X_{i,j}$, we can rewrite \eqref{eq:partial_ll} as follows.
\begin{align}
\label{eq:partial_ll2}
\pr(\yv_i | \xv_i) &=  \mathbbm{1}_{\{N_{Y,i} = 0\}}\cdot F(\infty,X^{\ast}_{i}) +  \sum_{s = 1}^{n} \mathbbm{1}_{\{N_{Y,i} = s\}}\cdot p(b_i = Y_{i, (s)} - X_{i,(s)}).
\end{align}
Therefore the log-likelihood of observing $Y$ given that $X \in \Xc_{Y, \nu, \gamma}$ is the original matrix takes the following form. 
\begin{align}
\Lc_{Y}(X) &= \sum_{i \in [d]}\log \pr(\yv_i | \xv_i) \nonumber \\
&= \sum_{i \in [d]}\Big(\mathbbm{1}_{\{N_{Y,i} = 0\}}\cdot \log F(\infty,X^{\ast}_{i}) +  \sum_{s = 1}^{n} \mathbbm{1}_{\{N_{Y,i} = s\}}\cdot \log p(Y_{i, (s)} - X_{i,(s)}) \Big).
\end{align}
In what follows, we work with a slightly modified quantity
\begin{align}
\overline{\Lc}_{Y}(X) = \Lc_{Y}(X) - \Lc_{Y}(0) = \sum_{i \in [d]}\Big(\mathbbm{1}_{\{N_{Y,i} = 0\}}\cdot \log \frac{F(\infty,X^{\ast}_{i})}{F(\infty,0)} +  \sum_{s = 1}^{n} \mathbbm{1}_{\{N_{Y,i} = s\}}\cdot \log \frac{p(Y_{i, (s)} - X_{i,(s)})}{p(Y_{i, (s)})} \Big). \nonumber
\end{align}
In order to recover the matrix $M$ from the observation matrix $Y$, we employ the natural maximum likelihood approach which is equivalent to the following.
\begin{equation}
\begin{aligned}
\label{eq:ml_real}
&\underset{X \in \R^{d \times n}}{\text{maximize}}
& & \overline{\Lc}_{Y}(X) ~~\text{subject to}~X \in \Xc_{Y, \nu, \gamma}.
\end{aligned}
\end{equation}

Define $\omega_{p,\gamma,\nu}$ to be such that $F(x,y) \ge \omega_{p,\gamma,\nu}$ for all $x,y\in [-\gamma, \gamma]$ with $|x-y|>\nu$. In what follows, we simply refer this quantity as $\omega_p$ as $\gamma$ and $\nu$ are clear from context.
The following result characterizes the performance of the program proposed in \eqref{eq:ml_real}.

\begin{theorem}
\label{thm:MLperformance_real}
Assume that $\|M\|_{\infty} \leq \gamma$ and the observation matrix $Y$ is related to $M$ according to \eqref{eq:sigmodel}. Let $\widehat{M}$ be the solution of the program specified in \eqref{eq:ml_real}, and the bias density function is $p(x)$ differentiable with bounded derivative.  Then, the following holds with probability at least $1 - \frac{1}{d + n}$.
\begin{align}
\|M - \widehat{M}\|^2_F  \leq  C_{0}  L_\gamma(p) \cdot \frac{\gamma d}{\beta_\gamma(p) \omega_p },
\end{align}
where, $C_0$ is a constant. The quantities $\beta_\gamma(p)$ and $L_\gamma(p)$ depend on the distribution of the bias and are defined in \eqref{eq:flat} and \eqref{eq:steep}, respectively.
\end{theorem}
The proof of Theorem~\ref{thm:MLperformance_real} crucially depends on the following lemma.
\begin{lemma}
\label{lem:dist_real}
Given the observation matrix $Y$ which is related to the matrix $M$ according to \eqref{eq:sigmodel}, let $\Xc_{Y, \nu, \gamma}$ be as defined in \eqref{eq:Xc_real}. Then, for any $X \in \Xc_{Y, \nu, \gamma}$, we have
\begin{align}
\avg\left[\overline{\Lc}_{Y}(M) - \overline{\Lc}_{Y}(X)\right] \geq \beta_\gamma(p) \omega_p \cdot \|M - X\|^2_F.
\end{align}
\end{lemma}
The proof of this lemma is delegated to the appendix. Now we are ready to prove Theorem~\ref{thm:MLperformance_real}.

\begin{proof}[Proof of Theorem~\ref{thm:MLperformance_real}]
Let $\widehat{M}$ be the solution of the program in \eqref{eq:ml_real}. In what follows, we use $\Xc$ as a short hand notation for $\Xc_{Y, \nu, \gamma}$. We have,
\begin{align}
\label{eq:stepP}
0 \le \overline{\Lc}_{Y}(\widehat{M}) - \overline{\Lc}_{Y}(M) &= \avg\left[\overline{\Lc}_{Y}(\widehat{M}) - \overline{\Lc}_{Y}(M)\right] - \big(\overline{\Lc}_{Y}(M) - \avg\left[\overline{\Lc}_{Y}(M)\right]\big) + \big(\overline{\Lc}_{Y}(\widehat{M}) - \avg\left[\overline{\Lc}_{Y}(\widehat{M})\right]\big) \nonumber \\
&\leq \avg\left[\overline{\Lc}_{Y}(\widehat{M}) - \overline{\Lc}_{Y}(M)\right]  + 2\sup_{X \in \Xc} \big\vert\overline{\Lc}_{Y}(X) - \avg\left[\overline{\Lc}_{Y}(X)\right]\big\vert,
\end{align}
which means,
\begin{align}
\avg\left[\overline{\Lc}_{Y}(M) - \overline{\Lc}_{Y}(\widehat{M})\right] \leq 2\sup_{X \in \Xc} \big\vert\overline{\Lc}_{Y}(X) - \avg\left[\overline{\Lc}_{Y}(X)\right]\big\vert.
\end{align}
We now employ Lemma~\ref{lem:dist_real} to obtain that
\begin{align}
\label{eq:perform1}
\beta_\gamma(p) \omega_p \cdot \|M - \widehat{M}\|^2_F  \leq 2\sup_{X \in \Xc} \big\vert\overline{\Lc}_{Y}(X) - \avg\left[\overline{\Lc}_{Y}(X)\right]\big\vert.
\end{align}

We now proceed to upper bound the right hand side of \eqref{eq:perform1}. It follows from the standard symmetrization trick~\cite{DevGL13} that, for any integer $t \geq 1$, we have
\begin{align}
\label{eq:RSstep1}
&\avg\left[\sup_{X \in \Xc} \big\vert \overline{\Lc}_{Y}(X) - \avg\left[\overline{\Lc}_{Y}(X)\right]\big\vert^t\right] \leq 2^t\cdot\avg\Bigg[\sup_{X \in \Xc}\Big\vert\sum_{i = 1}^{d}\varepsilon_i\cdot \Big(\mathbbm{1}_{\{N_{Y,i} = 0\}}\cdot \log \frac{F(\infty,X^{\ast}_{i})}{F(\infty,0)}~~+  \nonumber \\
&~~~~~~~~~~~~~~~~~~~~~~~~~~~~~~~~~~~~~~~~~~~~~~~~~~~~~~~~~~~~~~~~~~~~\sum_{s = 1}^{n} \mathbbm{1}_{\{N_{Y,i} = s\}}\cdot \log \frac{p(Y_{i, (s)} - X_{i,(s)})}{p(Y_{i, (s)})} \Big)\Big\vert^{t}\Bigg],
\end{align}
where $\{\varepsilon_i\}_{i \in [d]}$ are i.i.d. Rademacher random variables. 
Note that, for $x, \tilde{x} \in \R$, 
\begin{align*}
\log F(\infty,x) - \log F(\infty,\tilde{x}) & \le |x - \tilde{x}|\cdot \sup_{|u|\le \gamma} \frac{d (\log F(\infty, u))}{du}\\
& = |x - \tilde{x}|\cdot \sup_{|u|\le \gamma} \frac{d (\log \int_{-\infty}^{-u}p(y)dy )}{du} \le |x - \tilde{x}|\cdot L_\gamma(p),
\end{align*}
and
\begin{align*}
\log p(Y_{i, (s)} - x) - \log p(Y_{i, (s)} - \tilde{x}) & \le |x - \tilde{x}|\cdot \sup_{|u|\le \gamma} \frac{d p(u)}{du} \le |x - \tilde{x}|\cdot L_\gamma(p).
\end{align*}
 
At this point, 
we can combine the contraction principle with \eqref{eq:RSstep1} to obtain the following. 
\begin{align}
\label{eq:RCstep1}
&\avg\left[\sup_{X \in \Xc} \big\vert \overline{\Lc}_{Y}(X) - \avg\left[\overline{\Lc}_{Y}(X)\right]\big\vert^t\right] \nonumber \\
&~~~~~~~~~~~~~~~~~~~~~~~~\leq  2^t\cdot 2^t\cdot\avg\Bigg[(L_\gamma(p))^t\cdot\sup_{X \in \Xc}\Big\vert\sum_{i = 1}^{d}\varepsilon_i\cdot \Big(\mathbbm{1}_{\{N_{Y,i} = 0\}}\cdot X^{\ast}_{i} + \sum_{s = 1}^{n} \mathbbm{1}_{\{N_{Y,i} = s\}}\cdot X_{i,(s)} \Big)\Big\vert^{t}\Bigg] \nonumber \\
&~~~~~~~~~~~~~~~~~~~~~~~~~\overset{(i)}{\leq} 4^{t}\cdot \avg\Bigg[(L_\gamma(p))^t\cdot\sup_{X \in \Xc}\big(\sum_{i = 1}^{d}\varepsilon_i^2\big)^{t/2}\big(\sum_{i = 1}^{d}\big(\mathbbm{1}_{\{N_{Y,i} = 0\}}\cdot X^{\ast}_{i} + \sum_{s = 1}^{n} \mathbbm{1}_{\{N_{Y,i} = s\}}\cdot X_{i,(s)}\big)^2 \big)^{t/2}\Bigg] \nonumber \\
&~~~~~~~~~~~~~~~~~~~~~~~~~\overset{(ii)}{\leq} 4^{t}\cdot \avg\Bigg[(L_\gamma(p))^t\cdot d^{t/2}\big(d\gamma^2\big)^{t/2}\Bigg] = (4 L_\gamma(p) \cdot d \gamma )^t,
\end{align}
where $(i)$ and $(ii)$ follow from the Cauchy-Schwartz inequality and the fact that, for $X \in \Xc$, $\|X\|_{\infty} \leq \gamma$, respectively. Now using Markov's inequality, it follows from \eqref{eq:RCstep1} that
\begin{align}
\label{eq:RCstep2}
\mathbb{P}\left\{\sup_{X \in \Xc}\big\vert \overline{\Lc}_{Y}(X) - \avg\left[\overline{\Lc}_{Z}(X)\right]\big\vert\geq C_{0}  L_\gamma(p)\cdot{\gamma} d \right\} &\leq \frac{\avg\Big[\sup_{X \in \Xc}\big\vert \overline{\Lc}_{Z}(X) - \avg\left[\overline{\Lc}_{Z}(X)\right]\big\vert^{t}\Big]}{(C_{0} L_\gamma(p)\cdot{\gamma}d)^{t}} \nonumber \\
&\overset{(i)}{\leq} \left(\frac{4}{C_0}\right)^t \overset{(ii)}{\leq} \frac{1}{d+n},
\end{align}
where $(i)$ follows from \eqref{eq:RCstep1}; and $(ii)$ follows by setting $C_0 \geq \frac{4}{e}$ and $t = \log (d + n)$.
\end{proof}

\subsection{Recovering the network parameters}
\label{sec:col_space}

As established in Theorem~\ref{thm:MLperformance_real}, the program proposed in \eqref{eq:ml_real} recovers a matrix $\widehat{M} \in \Xc_{Y, \nu, \gamma}$ such that 
\begin{align}
\label{eq:perturb}
\|M - \widehat{M}\|_F \leq \sqrt{\frac{C_0\cdot L_\gamma(p)\cdot\gamma}{\beta_{\gamma}(p)\omega_p}} \sqrt{d}.
\end{align}
Let's denote the recovered matrix $\widehat{M}$ as $\widehat{M} = M + E$, where $E$ denotes the perturbation matrix that has bounded Frobenius norm (cf.~\eqref{eq:perturb}). Now the task of  recovering the parameters of single-layer $\relu$-network is equivalent to solving for $A$ given 
\begin{align}
\widehat{M} = M + E = AC + E.
\end{align}
In our setting where we have $A \in \R^{d \times k}$ and $C \in \R^{k \times n}$ with $d > k$ and $n > k$, $M$ is a low-rank matrix with its column space spanned by the columns of $A$. Therefore, as long as the generative model ensures that the matrix $M$ has its singular values sufficiently bounded away from $0$, we can resort to standard results from matrix-perturbation theory and output top $k$ left singular vectors of $\widehat{M}$ as an candidate for the orthonormal basis for the column space of $M$ or $A$. In particular, we can employ the result from \cite{YuWS14} which is stated in Appendix~\ref{sec:perturbation}. Let $U$ and $\widehat{U}$ be the top $k$ left singular vectors of $M$ and $\widehat{M}$, respectively. 
Note that, even without the perturbation we could only hope to recover the column space of $A$ (or the column space of $U$) and not the exact matrix $A$. Let $\sigma_k$, the smallest non-zero singular value of $M$, is at least $\delta > 0$. 
Then, it follows from Theorem~\ref{thm:yu} (cf.~Appendix~\ref{sec:perturbation}) and \eqref{eq:perturb} that there exists an orthogonal matrix $O \in \R^{k \times k}$ such that
\begin{align}
\|U - \widehat{U} O\|_F \leq \frac{2^{3/2}(2\sigma_1 + \|E\|)\cdot\min\{\sqrt{k}\|E\|, \|E\|_{F}\}}{\delta^2} \leq \frac{2^{3/2}(2\sigma_1 + \|E\|_F)\cdot\|E\|_{F}}{\delta^2},
\end{align}
which is a guarantee that the column space of $U$ is recovered within an error of $O(d)$ in Frobenius norm by the column space of $\widehat{U}$.


\section{Robust recovery in single-layer $\relu$-network}
\label{sec:robust}

We now explore the second fundamental question that arises in the context of reconstructing a signal vector belonging to the underlying generative model from its erroneous version. Recall that, we are given a vector $\vv\in \reals^d$, which is obtained by adding noise to a valid message vector $\yv\in \reals^d$ that is well modeled by a single-layer $\relu$-network, i.e.,
\begin{align}
\label{eq:obsmodel}
\vv = \yv + \ev + \wv = \relu(A\cv + \bv) + \ev + \wv.
\end{align}
Here, $\wv$ denotes the (dense) noise vector with bounded norm. On the other hands, the vector $\ev$ contains (potentially) large corruptions, also referred to as outliers. We assume the number of outliers $\|\ev\|_0$ to be bounded above by $s$. The robust recovery problem in $\relu$-networks corresponds to obtaining an estimate $\hat{\cv}$ of the true representation $\cv$ from the {\em corrupt} observation vector $\yv$ such that the distance between $\hat{\cv}$ and $\cv$ is small. A related problem of denoising in the presence of outliers only focuses on obtaining an estimate $\hat{\yv}$ which is close to the true message vector $\yv$. In the remainder of this paper, we focus on the setting where the weight matrix $A$ is a random matrix with i.i.d. entries, where each entry is distributed according to the standard Gaussian distribution. Furthermore, another crucial assumption is that the outlier vector is {\em oblivious} in its nature, i.e., the error vector is not picked in an adversarial manner\footnote{It is an interesting problem to extend our results to a setting with adversarial errors. However, we note that this problem is an active area of research even in the case of linear measurement, i.e, $\yv = A\cv + \ev + \wv$~\cite{BJK15, BJKK17}. We plan to explore this problem in future work.} given the knowledge of $A$ and $\cv$.

Note that \cite{Soltan17} study a problem which is equivalent to recovering the latent vector $\cv$ from the observation vector generated form a single-layer $\relu$-network without the presence of outliers. In that sense, our work is a natural generalization of the work in \cite{Soltan17} and presents a recovery method which is robust to errors as well. However, our approach significantly differs from that in \cite{Soltan17}, where the author analyze the convergence of the gradient descent method to the true representation vector $\cv$. In contrast, we rely on the recent work of Plan and Vershynin~\cite{PlanV16} to employ the LASSO method to recover the representation vector $\cv$ (and the Hamming error vector $\ev$). 

Given $\vv = \relu(A\cv + \bv) + \ev + \wv$, which corresponds to the corrupted non-linear observations of $\cv$, we try to fit a linear model to these observations by solving the following optimization problem\footnote{Note that this paper deals with a setup where number of observations is greater than the dimension of the signal that needs to be recovered, i.e., $d > k$. Therefore, we don't necessarily require the vector $\cv$ to belong to a restricted set, as done in the other version of the robust LASSO methods for linear measurements (see e.g.,~\cite{NgTran13}).}.
\begin{equation}
\begin{aligned}
\label{eq:rlasso}
&\underset{\cv \in \R^{k}, \ev \in \R^{d}}{\text{minimize}}
& & \frac{1}{2d}\|\vv - A\cv - \ev\|^2_2  + \lambda\|\ev\|_1.
\end{aligned}
\end{equation}
In the aforementioned formulation, the regularizer part is included to encourage the sparsity in the estimate vector. The following result characterizes the performance of our proposed program (cf.~\eqref{eq:rlasso}) in recovering the representation and the corruption vector.

\begin{theorem}
\label{thm:rlasso}
Let $A \in \R^{d \times n}$ be a random matrix with i.i.d. standard Gaussian random variables as its entires and $\vv$ satisfies 
\begin{align}
\label{eq:observe_vv}
\vv = \relu(A\cv^{\ast} + \bv) + \ev^{\ast} + \wv,
\end{align}
where $\|\cv^{\ast}\|_2 = 1$, $\|\ev^{\ast}\|_0 \le s$ and $\|\wv\|_{\infty} \leq \delta$. Let $\mu$ be defined as $\mu  = \avg[\relu(g + b)\cdot g]$, where $g$ is a standard Gaussian random variable and $b$ is a random variable that represents the bias in a coordinate in \eqref{eq:observe_vv}. Let $(\hat{\cv}, \hat{\ev})$ be the outcome of the program described in \eqref{eq:rlasso}. Then, with high probability, we have
\begin{align}
\label{eq:rlasso_perform}
\|\mu\cv^{\ast} - \hat{\cv}\|_2 + \frac{\|\ev^{\ast} - \hat{\ev}\|_2}{\sqrt{d}} \leq \tilde{C}\max\left\{\sqrt{\frac{{k\log k}}{d}}, \sqrt{\frac{s\log d}{d}}\right\},
\end{align}
where $\tilde{C}$ is a large enough absolute constant that depends on $\delta$. 
\end{theorem}
\begin{proof}
Assume that 
\begin{align}
\hat{\cv} = \mu \cv^{\ast} + \hv~~\text{and}~~\hat{\ev} = \ev^{\ast} + \fv.
\end{align}
Furthermore, for $\cv \in \R^{d}$ and $\ev \in \R^k$, we define
\begin{align}
\Lc(\cv, \ev) =  \frac{1}{2d}\|\yv - A\cv - \ev\|^2_2 + \lambda\|\ev\|_1.
\end{align}
Let $\Sc = \{i \in[d] : e^{\ast}_i \neq 0\}$ be the support of the vector $\ev^{\ast}$ such that $|\Sc| = s$. Given a vector $\av \in \R^d$ and set $\Tc \subseteq [d]$, we use $\av_{\Tc}$ to denote the vector obtained by restricting $\av$ to the indices belonging to $\Tc$. Note that
\begin{align}
\label{eq:Loss_diff}
\Lc(\hat{\cv}, \hat{\ev}) - \Lc(\mu \cv^{\ast}, \ev^{\ast}) & = \frac{1}{2d}\|\vv - \mu  A\cv^{\ast} - \ev^{\ast} - A\hv - \fv\|^2_2  + \lambda\|\ev^{\ast} + \fv\|_1 - \frac{1}{2d}\|\vv - \mu  A\cv^{\ast} - \ev^{\ast}\|^2_2 - \lambda\|\ev^{\ast}\|_1 \nonumber \\
& = \frac{1}{2d}\|A\hv + \fv\|^2_2 - \frac{1}{d}\langle \vv - \mu A\cv^{\ast} - \ev^{\ast}, A\hv + \fv\rangle + \lambda\big(\|(\ev^{\ast} + \fv^{\ast})_{\Sc}\|_1 + \|\fv_{\Sc^{C}}\|_1 - \|\ev^{\ast}\|_1\big)\nonumber \\
& \overset{(i)}{=} \frac{1}{2d}\|A\hv + \fv\|^2_2 - \frac{1}{d}\langle \relu(A\cv^{\ast} + \bv) - \mu A\cv^{\ast} + \wv, A\hv + \fv\rangle ~+  \nonumber \\
&~~~~~~\lambda\big(\|(\ev^{\ast} + \fv^{\ast})_{\Sc}\|_1 + \|\fv_{\Sc^{C}}\|_1 - \|\ev^{\ast}\|_1\big) \nonumber \\
& \overset{(ii)}{\geq} \frac{1}{2d}\|A\hv + \fv\|^2_2 - \frac{1}{d}\langle \relu(A\cv^{\ast} + \bv) - \mu A\cv^{\ast} + \wv, A\hv + \fv\rangle +\lambda\big(\|\fv_{\Sc^{C}}\|_1 - \|\fv_{\Sc}\|_1\big)
\end{align}
where $(i)$ and $(ii)$ follow from \eqref{eq:obsmodel} and the triangle inequality, respectively. Since $(\hat{\cv}, \hat{\ev})$ is solution to the program in \eqref{eq:rlasso}, we have 
\begin{align}
\Lc(\hat{\cv}, \hat{\ev}) - \Lc(\mu \cv^{\ast}, \ev^{\ast}) \leq 0.
\end{align}
By combining this with \eqref{eq:Loss_diff}, we obtain that
\begin{align}
\label{eq:standard}
\frac{1}{2d}\|A\hv + \fv\|^2_2  \leq \frac{1}{d}\cdot\langle \relu(A\cv^{\ast} + \bv) - \mu A\cv^{\ast} + \wv, A\hv + \fv\rangle + \lambda(\|\fv_{\Sc}\|_1 - \|\fv_{\Sc^{C}}\|_1)
\end{align}
We now complete the proof in two steps where we obtain universal lower and upper bounds on the left hand side and the right hand side of \eqref{eq:standard}, respectively, that hold with high probability. \\

\noindent\textbf{Upper bound on the RHS of \eqref{eq:standard}.}~Let's define
\begin{align}
\zv = \relu(A\cv^{\ast} + \bv) - \mu A\cv^{\ast}.
\end{align}
Note that 
\begin{align}
&\frac{1}{d}\cdot \langle \zv+\wv, A\hv + \fv\rangle + \lambda(\|\fv_{\Sc}\|_1 - \|\fv_{\Sc^{C}}\|_1)\big) \nonumber \\
&~~~~~~~~~~~~~~~~~~~~~~~~~~~~~= \frac{1}{d}\cdot \langle \zv + \wv, A\hv\rangle +  \big(\frac{1}{d}\cdot\langle \zv+ \wv, \fv\rangle + \lambda(\|\fv_{\Sc}\|_1 - \|\fv_{\Sc^{C}}\|_1)\big) \nonumber \\
&~~~~~~~~~~~~~~~~~~~~~~~~~~~~~\overset{(i)}{\leq} \frac{1}{d} \cdot \langle \zv+ \wv, A\hv\rangle + \big(\frac{1}{d}\cdot\|\zv+\wv\|_{\infty}\|\fv\|_1 +  \lambda(\|\fv_{\Sc}\|_1 - \|\fv_{\Sc^{C}}\|_1)\big)\nonumber \\
&~~~~~~~~~~~~~~~~~~~~~~~~~~~~~=  \frac{1}{d}\cdot \langle \zv+ \wv, A\hv\rangle + (\lambda + \frac{1}{d}\cdot\|\zv+\wv\|_{\infty} )\|\fv_{\Sc}\|_1 -  (\lambda - \frac{1}{d}\cdot\|\zv+\wv\|_{\infty})\|\fv_{\Sc^{C}}\|_1 \label{eq:pv_upperbound} 
\end{align}
where $(i)$ follows from the H\"older's inequality. We now employ \cite[Lemma 4.3]{PlanV16} to obtain that\footnote{In \cite{PlanV16}, Plan and Vershynin obtain the bound in terms of the Gaussian width~\cite{HDPbook} of the cone which the vector $\hv$ belongs to. However, in our setup where we do not impose any specific structure on $\cv^{\ast}$, this quantity is simply $O(\sqrt{k})$.}
\begin{align}
\label{eq:pv_main}
\sup_{\hv \in \R^{k}} \langle \zv, A\hv \rangle & \leq C\big(\sqrt{k}\sigma + \eta\big)\sqrt{d}\cdot\|\hv\|_2,
\end{align}
where $C$ is an absolute constant and  
\begin{align}
\label{eq:sigeta}
\sigma^2 &:= \avg\big[(\relu(g + b) - \mu  g)^2\big] \nonumber \\
\eta^2 &:= \avg\big[g^2\cdot(\relu(g + b) - \mu  g)^2\big]
\end{align}
with $g$ being a standard Gaussian random variable. Now we can combine \eqref{eq:pv_upperbound} and \eqref{eq:pv_main} to obtain the following.
\begin{align}
&\frac{1}{d}\cdot \langle \zv+\wv, A\hv + \fv\rangle +  \lambda\big(\|\fv_{\Sc}\|_1 - \|\fv_{\Sc^{C}}\|_1\big) \nonumber \\
&~~~~~\leq {C}\frac{\big(\sqrt{k}\sigma + \eta\big)}{\sqrt{d}}\cdot\|\hv\|_2 + \frac{\sqrt{k}}{{d}}\|A^T\wv\|_{\infty}\|\hv\|_2 + \left(\lambda + \frac{\|\zv+\wv\|_{\infty}}{d} \right)\|\fv_{\Sc}\|_1 -  \left(\lambda - \frac{\|\zv+\wv\|_{\infty}}{d}\right)\|\fv_{\Sc^{C}}\|_1 \label{eq:pv_midstep} \\
&~~~~~~\overset{(i)}{\leq}  {C}\frac{\big(\sqrt{k}\sigma + \eta\big)}{\sqrt{d}}\cdot\|\hv\|_2 + \frac{\sqrt{k}}{{d}}\|A^T\wv\|_{\infty}\|\hv\|_2 +  \sqrt{s}(\lambda + \frac{1}{d}\cdot\|\zv+\wv\|_{\infty} )\|\fv\|_2  \nonumber \\
&~~~~~~\overset{(ii)}{\leq}  {C}\frac{\big(\sqrt{k}\sigma + \eta\big)}{\sqrt{d}}\cdot\|\hv\|_2 + \frac{\sqrt{k}}{{d}}\|A^T\wv\|_{\infty}\|\hv\|_2 +  2\lambda \sqrt{s}\|\fv\|_2 \label{eq:pv_upperbound1},
\end{align}
where $(i)$ and $(ii)$ follow by setting $\lambda \geq 2\|\zv + \wv\|_{\infty}/d$ and using the fact that $\|\fv_{\Sc}\|_1 \leq \sqrt{s}\|\fv_{\Sc}\|_2 \leq \sqrt{s}\|\fv\|_2.$ We can further simplify the bound in \eqref{eq:pv_upperbound1} as follows.
\begin{align}
&\frac{1}{d}\cdot \langle \zv+\wv, A\hv + \fv\rangle +  \lambda\big(\|\fv_{\Sc}\|_1 - \|\fv_{\Sc^{C}}\|_1\big) \nonumber \\
&~~~~~~\leq  \max\left\{{C}\frac{\big(\sqrt{k}\sigma + \eta\big)}{\sqrt{d}} + \frac{\sqrt{k}}{{d}}\|A^T\wv\|_{\infty}, 2\lambda\sqrt{s}\sqrt{d}\right\} \left(\|\hv\|_2 +  \frac{\|\fv\|_2}{\sqrt{d}}\right). \label{eq:pv_upperbound2}
\end{align}
\\

\noindent\textbf{Lower bound on the LHS of \eqref{eq:standard}.}~By combining \eqref{eq:standard} and \eqref{eq:pv_midstep}, we get that
\begin{align}
\label{eq:restricted1}
&\frac{1}{2d}\|A\hv + \fv\|_2^2  \nonumber \\
&~~~~~\leq {C}\frac{\big(\sqrt{k}\sigma + \eta\big)}{\sqrt{d}}\cdot\|\hv\|_2 + \frac{\sqrt{k}}{{d}}\|A^T\wv\|_{\infty}\|\hv\|_2 + \left(\lambda + \frac{\|\zv+\wv\|_{\infty}}{d} \right)\|\fv_{\Sc}\|_1 -  \left(\lambda - \frac{\|\zv+\wv\|_{\infty}}{d}\right)\|\fv_{\Sc^{C}}\|_1 
\end{align} 
Note that we have picked $\lambda \geq 2\frac{\|\zv + \wv\|_{\infty}}{d}$. Since the left hand side of \eqref{eq:restricted1} is non-negative, we find that the tuple $(\hv, \fv)$ belongs to the following restricted set.
\begin{align}
\label{eq:Rset}
&(\hv, \fv) \in \Rc := \{\hv \in \R^k, \fv \in \R^d :  {\lambda}\cdot\|\fv_{\Sc^{C}}\|_1 \leq 2\left({C}\frac{\big(\sqrt{k}\sigma + \eta\big)}{\sqrt{d}} + \frac{\sqrt{k}}{{d}}\|A^T\wv\|_{\infty}\right)\|\hv\|_2 + {3\lambda}\|\fv_{\Sc}\|_1 \}.
\end{align}

As a result, in order to lower bound \eqref{eq:standard}, we lower bounding the following quantity for every $(\hv, \fv) \in \Rc$.
\begin{align}
\frac{1}{2d}\cdot\|A\hv + \fv\|^2_2.
\end{align}
Towards this, we employ Lemma~\ref{lem:NgTran} in Appendix~\ref{appen:robust}, which gives us that, for every $(\hv, \fv) \in \Rc$, with high probability, we have 
\begin{align}
\label{eq:pv_lower}
\frac{1}{2d}\cdot\|A\hv + \fv\|^2_2 \geq \frac{1}{128}\left(\|\hv\|_2  + \frac{\|\fv\|_2}{\sqrt{d}}\right)^2.
\end{align}

\noindent\textbf{Completing the proof}.~It follows from \eqref{eq:standard}, \eqref{eq:pv_upperbound2}, and \eqref{eq:pv_lower} that
$$
\frac{1}{128}\left(\|\hv\|_2  + \frac{\|\fv\|_2}{\sqrt{d}}\right)^2 \leq  \max\left\{{C}\frac{\big(\sqrt{k}\sigma + \eta\big)}{\sqrt{d}} + \frac{\sqrt{k}}{{d}}\|A^T\wv\|_{\infty}, 2\lambda\sqrt{s}\sqrt{d}\right\} \left(\|\hv\|_2 +  \frac{\|\fv\|_2}{\sqrt{d}}\right)
$$
or 
\begin{align}
\label{eq:final_hf}
\|\hv\|_2 + \frac{\|\fv\|_2}{\sqrt{d}} \leq  128\max\left\{{C}\frac{\big(\sqrt{k}\sigma + \eta\big)}{\sqrt{d}} + \frac{\sqrt{k}}{{d}}\|A^T\wv\|_{\infty}, 2\lambda\sqrt{s}\sqrt{d}\right\}.
\end{align}
Now using the fact that $\|\wv\|_{\infty} \leq \delta$ and $A$ is an i.i.d. standard Gaussian matrix, we can obtain the following bound from \eqref{eq:final_hf}, which holds with high probability. 
\begin{align}
\|\hv\|_2 + \frac{\|\fv\|_2}{\sqrt{d}} \leq \tilde{C}\max\left\{\sqrt{\frac{{k\log k}}{d}}, \sqrt{\frac{s\log d}{d}}\right\},
\end{align}
where $\tilde{C}$ is a large enough absolute constant.
\end{proof}


\paragraph{Acknowledgements.} This research is supported in part by NSF awards CCF 1642550, CCF 1618512 and CAREER award 1642658.


\bibliographystyle{plainnat}
\bibliography{ReLU}


\appendix 

\section{Results on Matrix Perturbation}
\label{sec:perturbation}

Let $M$ be a $d \times n$ matrix, where without loss of generality we assume that $d \geq n$. Let $M$ have the following singular value decomposition.
\begin{align}
M = U \Sigma V^T, 
\end{align}
where $\Sigma = {\rm Diag}\left(\sigma_1, \sigma_2,\ldots, \sigma_n\right)$ is the diagonal matrix with the singular values of $M$ as its diagonal entries. Let $\widehat{M} = M + E$ be the matrix which is obtained by perturbing the original matrix $M$ by an error matrix $E$. Let $\widehat{M}$ have the following singular value decomposition.
\begin{align}
\widehat{M} = \widehat{U} \widehat{\Sigma} \widehat{V}^T,
\end{align}
where $\widehat{\Sigma} = {\rm Diag}\left(\widehat{\sigma}_1,\widehat{\sigma}_2,\ldots, \widehat{\sigma}_n\right)$ is the diagonal matrix comprising the singular values of the perturbed matrix $\widehat{M}$. Let $\gamma_1 \geq \gamma_2 \geq \cdots \geq \gamma_n$ be the singular values of the matrix $U^T\widehat{U}$. Define, 
\begin{align}
\theta_i = \cos^{-1}\gamma_i~~\forall i \in [n]~~\text{and}~~\Theta(U, \widehat{U}) = {\rm Diag}\big(\theta_1, \theta_2,\ldots, \theta_n\big).
\end{align}
Note that $\{\theta_i\}_{i \in [n]}$ are referred to as the {\em canonical angles} between $\Uc$ and $\widehat{\Uc}$, the subspaces spanned by the columns of the matrices $U$ and $\widehat{U}$, respectively. It is common to use $\|\sin \Theta(U, \widehat{U})\|_F$ as a distance measure between $\Uc$ and $\widehat{\Uc}$.

In \cite{YuWS14}, Yu et al. present the following result which bounds the distance between the subspaces spanned by the singular vectors of the original matrix $M$ and the perturbed matrix $\widehat{M}$, respectively\footnote{Here, we state a special case of the result from \cite{YuWS14}. See \cite[Theorem 4]{YuWS14} for the statement of the general result.}. 

\begin{theorem}
\label{thm:yu}
Let $M, \widehat{M} = M + E \in \R^{d \times n}$ have singular values $\sigma_1 \geq \cdots \sigma_{\min\{d, n\}}$ and $\widehat{\sigma}_1 \geq \cdots \widehat{\sigma}_{\min\{d, n\}}$, respectively. Fix $1 \leq r \leq \rank(A)$ and assume that $$\sigma^2_{r} - \sigma^2_{r+1} \geq 0.$$ Let $U = (\uv_1, \uv_2,\ldots, \uv_{r}) \in \R^{d \times r}$ and $\widehat{U} = (\widehat{\uv}_1, \widehat{\uv}_2,\ldots, \widehat{\uv}_r) \in \R^{d \times r}$ contain the left singular vectors associated with the $r$ leading singular values of $M$ and $\widehat{M}$, respectively. Then, 
\begin{align}
\|\sin \Theta(U, \widehat{U})\|_F \leq \frac{2(2\sigma_1 + \|E\|)\cdot\min\{\sqrt{r}\|E\|, \|E\|_{F}\}}{\sigma_r^2 - \sigma_{r+1}^2}.
\end{align}
Moreover, there exists an orthogonal matrix $O \in \R^{r \times r}$ such that
\begin{align}
\|U - \widehat{U} O\|_F \leq \frac{2^{3/2}(2\sigma_1 + \|E\|)\cdot\min\{\sqrt{r}\|E\|, \|E\|_{F}\}}{\sigma_r^2 - \sigma_{r+1}^2}.
\end{align}
\end{theorem}

\section{Proofs of Section~\ref{sec:ML_real}}
\label{appen:ML_real}

\begin{lemma}[\textbf{Defining the distance measure}] 
Given the observation matrix $Y$ which is related to the matrix $M$ according to \eqref{eq:sigmodel}, let $\Xc_{Y, \nu, \gamma}$ be as defined in \eqref{eq:Xc_real}. Then, for any $X \in \Xc_{Y, \nu, \gamma}$, we have
\begin{align}
\avg\left[\overline{\Lc}_{Y}(M) - \overline{\Lc}_{Y}(X)\right] \geq \beta_{\gamma}(p) \omega_p \cdot \|M - X\|^2_F.
\end{align}
\end{lemma}

\begin{proof} First, we recall our notation that given the original matrix $M$, for $i \in [d]$ and $j \in [n]$, $M_{i,(j)}$ denotes the $j$-th largest element of the $i$-th row of $M$, i.e., for $i \in [d]$,
\begin{align*}
M_{i, (1)} \geq M_{i, (2)} \geq \cdots \geq M_{i, (n)}.
\end{align*}
We define, $M^{\ast}_{i} = \max_{j \in [n]}M_{i,j}.$
Thus,
\begin{align}
\label{eq:Dexp_real}
&\avg\left[\overline{\Lc}_{Y}(M) - \overline{\Lc}_{Y}(X)\right] \nonumber \\
&~~~~~= \sum_{i = 1}^{d}\avg\left[\mathbbm{1}_{\{N_{Y,i} = 0\}}\cdot \log \frac{F(\infty,M^{\ast}_{i})}{F(\infty,X^{\ast}_{i})} +  \sum_{s = 1}^{n} \mathbbm{1}_{\{N_{Y,i} = s\}}\cdot \log \frac{p(Y_{i, (s)} - M_{i,(s)})}{p(Y_{i, (s)} - X_{i,(s)})}\right] \nonumber \\
&~~~~~= \sum_{i = 1}^{d} \Big( F(\infty,M^{\ast}_{i})\cdot \log \frac{F(\infty,M^{\ast}_{i})}{F(\infty,X^{\ast}_{i})} + \sum_{s = 1}^{n-1} \int_{-M_{i,(s)}}^{-M_{i,(s+1)}}p(b) \cdot \log \frac{p(b)}{p(b +  M_{i,(s)} - X_{i,(s)})}d b~~+ \nonumber \\
&~~~~~~~~~~~~~~~~~~~~~~~~~~\int_{-M_{i,(n)}}^{\infty}p(b) \cdot \log \frac{p(b)}{p(b +  M_{i,(n)} - X_{i,(n)})}d b \Big),
\end{align}
where $p : \R \to \R$ represents the probability density function of the bias RV and $F$ is defined as
$$
F(x_1, x_2) = \mathbb{P}(-x_1 \leq b \leq - x_2).  
$$
Given the matrices, $M, X \in \Xc_{Z, \nu, \gamma}$, we define a new (density) function $g$ as follows.
\begin{align}
g(u) = \begin{cases}
p(u + M_{i,(s)} - X_{i,(s)}) & \text{if}~~u \in (-M_{i,(s)}, - M_{i,(s+1)}]~~\text{for}~~s \in [n-1] \\
p(u + M_{i,(n)} - X_{i,(n)} ) & \text{if}~~u \in [-M_{i,(n)}, \infty).
\end{cases}
\end{align}
Recall that for $x \in \R$, we have $x \leq e^x - 1$. For $x = \log y$, this gives us that 
\begin{align}
\label{eq:basic_ineq}
\log y \leq y - 1.
\end{align}
In particular, for $b \in \R$, employing \eqref{eq:basic_ineq} with $y = \sqrt{\frac{g(b)}{p(b)}}$, we get that
\begin{align*}
p(b)\cdot \log \sqrt{\frac{g(b)}{p(b)}} \leq p(b) \cdot \left(\sqrt{\frac{g(b)}{p(b)}} - 1\right) = \left(\sqrt{p(b)g(b)} - p(b)\right)
\end{align*}
or 
\begin{align}
\label{eq:basic_1}
p(b)\cdot \log \frac{p(b)}{g(b)} \geq 2\cdot \left(p(b) - \sqrt{p(b)g(b)}\right).
\end{align}
By using \eqref{eq:basic_1}, for every $i \in [d]$, we obtain that
\begin{align}
\label{eq:basic_2}
&\sum_{s = 1}^{n-1} \int_{-M_{i,(s)}}^{-M_{i,(s+1)}}p(b) \cdot \log \frac{p(b)}{p(b +  M_{i,(s)} - X_{i,(s)})}d b + \int_{-M_{i,(n)}}^{\infty}p(b) \cdot \log \frac{p(b)}{p(b +  M_{i,(n)} - X_{i,(n)})}d b \nonumber \\
&~~~~~~~~ = \int_{-M_{i,(1)}}^{\infty} p(b) \cdot \log \frac{p(b)}{g(b)} d b \nonumber \\
&~~~~~~~~\geq  \int_{-M_{i,(1)}}^{\infty} 2\cdot \left( p(b) - \sqrt{p(b)g(b)} \right) \nonumber \\
&~~~~~~~~=  2\cdot \Big( 1 - F(\infty,M_{1,(1)}) -  \int_{-M_{i,(1)}}^{\infty} \sqrt{p(b)g(b)} d b \Big)\nonumber \\
&~~~~~~~~ = \left(1 - F(\infty,M^{\ast}_{i}) \right) +  \left(1 - F(\infty,X^{\ast}_{i}) \right)  -  \int_{-M_{i,(1)}}^{\infty} 2 \cdot \sqrt{p(b)g(b)} d b ~~~+ \nonumber \\
&~~~~~~~~~~~~~~~~~~~~~~~~~~~ \left(1 - F(\infty,M^{\ast}_{i}) \right) - \left(1 - F(\infty,X^{\ast}_{i}) \right) \nonumber \\
&~~~~~~~~ =   \int_{-M_{i,(1)}}^{\infty} \left( \sqrt{p(b)} - \sqrt{g(b)}\right)^2 d b  +  \left(1 - F(\infty,M^{\ast}_{i}) \right) - \left(1 - F(\infty,X^{\ast}_{i}) \right)
\end{align}
For $i \in [d]$, we now employ \eqref{eq:basic_ineq} with $y = \frac{F(\infty,X^{\ast}_{i})}{F(\infty,M^{\ast}_{i})}$ to obtain the following.
\begin{align*}
F(\infty,M^{\ast}_{i})\cdot \log \frac{F(\infty,X^{\ast}_{i})}{F(\infty,M^{\ast}_{i})} \leq F(\infty,M^{\ast}_{i}) \cdot \left( \frac{F(\infty,X^{\ast}_{i})}{F(\infty,M^{\ast}_{i})} - 1\right) =  {F(\infty,X^{\ast}_{i})} - {F(\infty,M^{\ast}_{i})}.
\end{align*}
or
\begin{align}
\label{eq:basic_3} 
&F(\infty,M^{\ast}_{i})\cdot \log \frac{F(\infty,M^{\ast}_{i})}{F(\infty,X^{\ast}_{i})}  + {F(\infty,X^{\ast}_{i})} - {F(\infty,M^{\ast}_{i})} \nonumber \\
&~~~~~~~~~ = F(\infty,M^{\ast}_{i})\cdot \log \frac{F(\infty,M^{\ast}_{i})}{F(\infty,X^{\ast}_{i})}  + \left(1 - {F(\infty,M^{\ast}_{i})} \right)  - \left(1 - {F(\infty,X^{\ast}_{i})}\right) \geq 0 
\end{align}
By combining \eqref{eq:basic_2} and \eqref{eq:basic_3}, we obtain that
\begin{align}
\label{eq:basic_4}
& F(\infty,M^{\ast}_{i})\cdot \log \frac{F(\infty,M^{\ast}_{i})}{F(\infty,X^{\ast}_{i})} + \sum_{s = 1}^{n-1} \int_{-M_{i,(s)}}^{-M_{i,(s+1)}}p(b) \cdot \log \frac{p(b)}{p(b +  M_{i,(s)} - X_{i,(s)})}d b~~~+ \nonumber \\
&~~~~~~~~\int_{-M_{i,(n)}}^{\infty}p(b) \cdot \log \frac{p(b)}{p(b +  M_{i,(n)} - X_{i,(n)})}d b \nonumber \\
&~~~~~~~~~~~~~~~~~~~~~~~~~~~~~\geq  \int_{-M_{i,(1)}}^{\infty} \left( \sqrt{p(b)} - \sqrt{g(b)}\right)^2 d b \nonumber \\
&~~~~~~~~~~~~~~~~~~~~~~~~~~~~~ =  \sum_{s = 1}^{n-1} \int_{-M_{i,(s)}}^{-M_{i,(s+1)}} \left( \sqrt{p(b)} - \sqrt{p(b + M_{i,(s)} - X_{i,(s)})}\right)^2 d b~~~~+ \nonumber \\
&~~~~~~~~~~~~~~~~~~~~~~~~~~~~~~~~~\int_{-M_{i,(n)}}^{\infty}  \left( \sqrt{p(b)} - \sqrt{p(b + M_{i,(n)} - X_{i,(n)})}\right)^2 d b \nonumber \\
&~~~~~~~~~~~~~~~~~~~~~~~~~~~~~\overset{(i)}{=} \sum_{s = 1}^{n-1} \int_{-M_{i,(s)}}^{-M_{i,(s+1)}} \left( \frac{p'(\xi_{s})}{2 \sqrt{p(\xi_{s})}}\big(M_{i,(s)} - X_{i,(s)}\big)\right)^2 d b~~~~+ \nonumber \\
&~~~~~~~~~~~~~~~~~~~~~~~~~~~~~~~~~\int_{-M_{i,(n)}}^{\infty}  \left( \frac{p'(\xi_{n})}{2 \sqrt{p(\xi_{n})}}\big(M_{i,(n)} - X_{i,(n)}\big)\right)^2 d b \nonumber \\ 
&~~~~~~~~~~~~~~~~~~~~~~~~~~~~~\overset{(ii)}{\geq} \beta_{\gamma}(p) \cdot \left(\sum_{s = 1}^{n-1} \int_{-M_{i,(s)}}^{-M_{i,(s+1)}} \big(M_{i,(s)} - X_{i,(s)}\big)^2 d b + \int_{-M_{i,(n)}}^{\infty} \big(M_{i,(n)} - X_{i,(n)}\big)^2 d b \right) \nonumber \\
&~~~~~~~~~~~~~~~~~~~~~~~~~~~~~\overset{(iii)}{\geq} \beta_{\gamma}(p) \omega_{p} \cdot \sum_{s = 1}^{n} \big(M_{i,(s)} - X_{i,(s)}\big)^2,
\end{align}
where $(i)$ follows from the Mean Value Theorem with suitable $\{\xi_{i}\} \subset \R$ and $(ii)$ follows by assuming that we have 
$$
 \frac{p'(u)}{2 \sqrt{p(u)}} \geq \sqrt{\beta_{\gamma}(p)}~~\text{for all}~|u|\le \gamma.
$$
Since $M \in  \Xc_{Z, \nu, \gamma}$, we have $|M_{i,(n)}| \leq \gamma$ and $|M_{i,(s)} - M_{i,(s+1)}| \geq \nu$. The step $(iii)$ follows from the assumption that 
$$
F(x, y) \geq \omega_p~~\text{for all $(x, y)$~such that}~|x - y| \geq \nu.
$$
By combining \eqref{eq:Dexp_real} with \eqref{eq:basic_4}, we now obtain that
\begin{align*}
\avg\left[\overline{\Lc}_{Y}(M) - \overline{\Lc}_{Y}(X)\right]  \geq \sum_{i = 1}^{d} \beta_{\gamma}(p) \omega_{p} \cdot \sum_{s = 1}^{n} \big(M_{i,(s)} - X_{i,(s)}\big)^2 = \beta_{\gamma}(p) \omega_p \cdot \|M - X\|^2_F.
\end{align*} 
\end{proof}

\section{Proofs of Section~\ref{sec:robust}}
\label{appen:robust}

Here we state the special form of a result that was obtained in~\cite{NgTran13} for the generals setting, where one may potentially require the vector $\cv$ to be sparse as well.

\begin{lemma}
\label{lem:NgTran}
Let $A \in \R^{d \times n}$ be random matrix that has i.i.d. standard Gaussian entries. Furthermore, let $\Rc \subset \R^{k} \times \R^{d}$ be as defined in \eqref{eq:Rset}. Then, with probability at least $1 - c\exp(-\tilde{c} d)$, we have 
\begin{align}
\frac{1}{2d}\cdot\|A\hv +\fv\|_2 \geq \frac{1}{128}\left(\|\hv\| +\frac{\|\fv\|}{\sqrt{d}}\right)^2~~\text{for all}~(\hv, \fv) \in \Rc. 
\end{align}
Here, $c, \tilde{c} > 0 $ are absolute constants.
\end{lemma}

\begin{proof}
Note that 
\begin{align}
\label{eq:pv_lowerbound1}
\|A\hv + \fv\|_2^2 = \|A\hv\|_2^2 + \|\fv\|_2^2 + 2\langle A\hv, \fv\rangle.
\end{align}
For a  $d \times k$ matrix with i.i.d. Gaussian entries, there exists constants $c_1$ and $c_2$ such that with probability at least $1 - c_1\exp(-c_2d)$, we have 
\begin{align}
\frac{1}{\sqrt{d}}\|A\hv\|_2 \geq \frac{1}{4}\|\hv\|_2.
\end{align}
Therefore, with probability at least $1 - c_1 \exp(-c_2 d)$, we have
\begin{align}
\label{eq:pv_lowerbound2}
\|A\hv\|_2^2 + \|\fv\|_2^2 \geq \frac{d}{16}\|\hv\|^2_2 + \|\fv\|_2^2 \geq \frac{d}{16}(\|\hv\|_2^2 + \|\fv\|^2_2/d) \geq \frac{d}{32}\left(\|\hv\|_2 + \frac{\|\fv\|_2}{\sqrt{d}}\right)^{2}.
\end{align}
Next, we focus on obtaining an upper bound on 
$$
\frac{1}{d}|\langle A\hv, \fv\rangle|.
$$
Towards this we partition the set $[d]$ into $r$ blocks $\Sc_1 = \Sc, \Sc_2, \ldots, \Sc_r$ such that $|\Sc_2| = \cdots |\Sc|_r  = s' \geq |\Sc| = s$. Here, $\Sc_2$ refers to the set of indices of $s'$ largest entries (in terms of absolute value) of $\fv_{\Sc^{C}}$; $\Sc_3$ corresponds to the set of indices of the next $s'$ largest entires of $\fv_{\Sc^{C}}$; and so on. Now, we have
\begin{align}
\label{eq:correlation1}
\frac{1}{d}|\langle A\hv, \fv\rangle| \leq \frac{1}{d}\sum_{i = 1}^{r}|\langle A_{\Sc_i}\hv, \fv_{\Sc_i}\rangle| \leq \frac{1}{d}\max_{i}\|A_{\Sc_i}\|_2\|\hv\|_2 \sum_{i = 1}^{r}\|\fv_{\Sc_i}\|_2
\end{align}
In \cite[Appendix]{NgTran13}, Nguyen and Tran show that, with probability at least $1 - 2\exp(-\tau^2s'/2)$, for a set $\Sc'$ with $|\Sc'| = s'$, we have
\begin{align}
\|A_{\Sc'}\|_2 \leq \left( \sqrt{k} + \sqrt{s'} + \tau\sqrt{s'} \right).
\end{align}
By setting $\tau = \tau'\sqrt{\frac{d}{s'}}$ and taking the union bound over all the subsets of $[d]$ of size $s'$, 
\begin{align}
\label{eq:correlation2}
\|A_{\Sc'}\|_2 \leq \left( \sqrt{k} + \sqrt{s'} + \tau\sqrt{s'} \right)~~\forall~\Sc' \subset [d]~\text{such that}~|\Sc'| = s'
\end{align}
holds with probability at least 
$$
1 - {d \choose s'}\exp(-\tau'^2d/2) \geq 1 - \left(\frac{ed}{s'}\right)^{s'}\exp(-\tau'^2d/2).
$$
{Assuming that $s'\log(d/s') \leq c_3 d$}, the aforementioned probability is at least 
$$
1 - \exp(-(\tau'^2/2 - c_3)d).
$$
On the other hand, we have
\begin{align}
\label{eq:correlation3}
\sum_{i = 1}^{r}\|\fv_{\Sc_i}\|_2 & \overset{(i)}{\leq} 2\|\fv\|_2 + \sum_{i = 3}^{r}\|\fv_{\Sc_i}\|_2 \nonumber \\
&\overset{(ii)}{\leq} 2\|\fv\|_2 + \frac{1}{\sqrt{s'}}\|\fv_{\Sc^{c}}\|_1 \nonumber \\
&\overset{(iii)}{\leq} 2\|\fv\|_2 + \frac{2}{\lambda\sqrt{s'}}\left({C}\frac{\big(\sqrt{k}\sigma + \eta\big)}{\sqrt{d}} + \frac{\sqrt{k}}{{d}}\|A^T\wv\|_{\infty}\right)\|\hv\|_2 + \frac{3}{\sqrt{s'}}\|\fv_{\Sc}\|_1 \nonumber \\
&\overset{(iv)}{\leq} 5\|\fv\|_2 + \frac{2}{\lambda\sqrt{s'}}\left({C}\frac{\big(\sqrt{k}\sigma + \eta\big)}{\sqrt{d}} + \frac{\sqrt{k}}{{d}}\|A^T\wv\|_{\infty}\right)\|\hv\|_2,
\end{align}
where $(i)$ follows from the fact that $\|\fv_{\Sc_1}\|_2 \leq \|\fv_{\Sc_2}\|_2 \leq \|\fv\|_2$; $(ii)$ follows from a standard bound given in \cite{CandesRT06}; $(iii)$ follows from the fact that $\fv$ belongs to the set $\Rc$ defined in \eqref{eq:Rset}; and $(iv)$ is a consequence of the loose bound $\|\fv_{\Sc}\|_1 \leq \sqrt{s}\|\fv_{\Sc}\|_2 \leq \sqrt{s'}\|\fv\|_2$. Next, we use the fact that $\lambda \geq 2 \|\zv+ \wv\|_{\infty}/d$, it follows from \eqref{eq:correlation3} that
\begin{align}
\label{eq:correlation4}
\sum_{i = 1}^{r}\|\fv_{\Sc_i}\|_2 &\leq 5\|\fv\|_2 + \frac{d}{\|\zv + \wv\|_{\infty}\sqrt{s'}}\left({C}\frac{\big(\sqrt{k}\sigma + \eta\big)}{\sqrt{d}} + \frac{\sqrt{k}}{{d}}\|A^T\wv\|_{\infty}\right)\|\hv\|_2 \nonumber \\
&= 5\sqrt{d}\left(\frac{\|\fv\|_2}{\sqrt{d}} +  \frac{\sqrt{d}}{5\|\zv + \wv\|_{\infty}\sqrt{s'}}\left({C}\frac{\big(\sqrt{k}\sigma + \eta\big)}{\sqrt{d}} + \frac{\sqrt{k}}{{d}}\|A^T\wv\|_{\infty}\right)\|\hv\|_2 \right) \nonumber \\
&\leq {5\sqrt{d}\left(\frac{\|\fv\|_2}{\sqrt{d}} + \|\hv\|_2\right)}
\end{align}
By combining \eqref{eq:correlation1}, \eqref{eq:correlation2}, and \eqref{eq:correlation4}, with probability at least $1 - c_4\exp(-c_5d)$, we obtain that
\begin{align}
\label{eq:correlation5}
\frac{1}{d}|\langle A\hv, \fv\rangle| &\leq  \frac{1}{d}  \left( \sqrt{k} + \sqrt{s'} + \tau\sqrt{s'} \right) \|\hv\|_2 {\cdot 5\sqrt{d}\left(\frac{\|\fv\|_2}{\sqrt{d}} + \|\hv\|_2\right)} \nonumber \\
& \leq \frac{5}{\sqrt{d}}  \left( \sqrt{k} + \sqrt{s'} + \tau\sqrt{s'} \right) {\left(\frac{\|\fv\|_2}{\sqrt{d}} + \|\hv\|_2\right)^2} \nonumber \\
& \overset{(i)}{\leq} \frac{1}{128}{\left(\frac{\|\fv\|_2}{\sqrt{d}} + \|\hv\|_2\right)^2},
\end{align}
where $(i)$ follows from large enough $d$. By combining \eqref{eq:pv_lowerbound1}, \eqref{eq:pv_lowerbound2}, and \eqref{eq:correlation5}, we obtain that, for every $(\hv, \fv) \in \Rc$, 
\begin{align}
\frac{1}{2d}\cdot\|A\hv +\fv\|^2_2 \geq \frac{1}{128}\left(\|\hv\| +\frac{\|\fv\|_2}{\sqrt{d}}\right)^2
\end{align}
holds with probability at least $1 - c \exp(-\tilde{c} d)$, with absolute constants $c, \tilde{c} > 0$.
\end{proof}

\end{document}